\documentclass{article}

\usepackage[final]{nips_2018}
\usepackage{natbib}
\usepackage[utf8]{inputenc} 
\usepackage[T1]{fontenc}    
\usepackage{hyperref}       
\usepackage{url}            
\usepackage{booktabs}       
\usepackage{amsfonts}       
\usepackage{nicefrac}       
\usepackage{microtype}      
\usepackage{lipsum}

\usepackage{times}
\usepackage{epsfig}
\usepackage{graphicx}
\usepackage{epstopdf}
\usepackage{amsmath}
\usepackage{amssymb}
\usepackage{bbm, dsfont}

\usepackage{amsthm}
\newtheorem{theorem}{Theorem}

\newtheorem{lemma}{Lemma}
\newtheorem{proposition}{Proposition}

\usepackage{algorithm}
\usepackage{algorithmic}
\usepackage{empheq}
\usepackage{cleveref}

\title{Limited Gradient Descent: Learning With Noisy Labels}

\author{
 Yi~Sun \qquad  Yan~Tian \qquad  Yiping~Xu \qquad  Jianxiang~Li\\
 School of Electronic Information and Communications \\
 Huazhong University of Science and Technology\\
 \texttt{yi\_sun@hust.edu.cn, tianyan@hust.edu.cn} \\
 \texttt{xuyiping@hust.edu.cn, jianxiang\_li@hust.edu.cn}
}

\begin{document}
\maketitle

\begin{abstract}
Label noise may affect the generalization of classifiers, and the effective learning of main patterns from samples with noisy labels is an important challenge. Recent studies have shown that deep neural networks tend to prioritize the learning of simple patterns over the memorization of noise patterns. This suggests a possible method to search for the best generalization that learns the main pattern until the noise begins to be memorized. Traditional approaches often employ a clean validation set to find the best stop timing of learning, i.e., early stopping. However, the generalization performance of such methods relies on the quality of validation sets. Further, in practice, a clean validation set is sometimes difficult to obtain. To solve this problem, we propose a method that can estimate the optimal stopping timing without a clean validation set, called limited gradient descent. We modified the labels of a few samples in a noisy dataset to obtain false labels and to create a reverse pattern. By monitoring the learning progress of the noisy and reverse samples, we can determine the stop timing of learning. In this paper, we also theoretically provide some necessary conditions on learning with noisy labels. Experimental results on CIFAR-10 and CIFAR-100 datasets demonstrate that our approach has a comparable generalization performance to methods relying on a clean validation set. Thus, on the noisy Clothing-1M dataset, our approach surpasses methods that rely on a clean validation set.
\end{abstract}

\section{Introduction}
Noisy labels tend to affect the generalization performance of machine learning. Errors are often inevitable in manual annotation. Moreover, many datasets are constructed by crawling images and labels from websites, and these often contain a considerable number of noisy labels (e.g., YFCC100M~\cite{thomee2015yfcc100m}, Clothing1M~\cite{xiao2015learning}). Therefore, research on learning with noisy labels has great importance.

Deep neural networks (DNNs) have been applied to achieve breakthroughs in many fields. Many DNN-based methods have been proposed for learning with noisy labels. However, owing to the powerful fitting ability, DNNs may even memorize noise~\cite{zhang2016understanding}, which might hamper the generalization of the main pattern (pattern of interest).
However, a recent work~\cite{arpit2017closer} further revealed that DNNs prioritize the learning of simple patterns over the memorization of noise. During training, the generalization performance of the main pattern increases first and then decreases.

Traditional approaches~\cite{xiao2015learning,zhang2017mixup,tanaka2018joint} often employ a clean validation set to identify the best stop timing for learning, i.e., early stopping. However, these methods are almost sensitive to the validation sets. The quality of the validation set directly affects the generalization performance. In fact, it is cumbersome to produce a high-quality validation set, and in practice applications, clean validation sets are sometimes not readily available.

In this work, we focus on the learning of noisy labels without involving clean samples. To estimate the best stop timing of training, we propose a method called limited gradient descent (LGD) based on the characteristic that a classifier learns the main pattern until the noise pattern begins to be memorized. This method hopes to monitor the learning progresses of the main and noise patterns. Unfortunately, samples of different patterns cannot be initially distinguished.

Thus, we randomly select a few samples from a noisy training set as the reverse pattern, which is mutually exclusive to the main pattern. Specifically, we shift the labels of the selected samples as reverse labels (as in $label+1$). It can be proved that the reverse labels are almost false. Note that the samples of the main pattern are still unknown. We can obtain the training accuracies for the two parts of the samples: the reverse samples and leftover noisy samples. At the early stage of training, the accuracy of the leftover samples increases because the main pattern is learned first, and the accuracy of the reverse samples does not increase (or may even decrease). We could therefore monitor the ratio of the two accuracies to predict the best generalization.

We evaluate the performance on the CIFAR-10, CIFAR-100, and Clothing-1M datasets. The experimental results on CIFAR-10 and CIFAR-100 show that although the accuracies with our method are comparable to those from corresponding traditional methods, the variances of the experimental results are significantly reduced, which shows that our method has better robustness. We further demonstrate state-of-the-art performance on the noisy real-world Clothing-1M dataset.

The main contributions of the present study are as follows. First, we propose a weakly supervised method called limited gradient descent (LGD) that can learn the main pattern to the maximum extent possible from noisy labels. Second, we prepare a few samples with false labels for training, which no study has attempted thus far to the best of our knowledge. Third, we theoretically prove some necessary conditions on LGD learning with noisy labels. Lastly, our method is free of models; thus, it can be applied to most DNNs and loss functions based on the stochastic gradient descent (SGD) optimization.

\section{Related Works}
Learning with noisy labels has been a long-standing problem in machine learning, which can be traced back to the 1980s~\cite{angluin1988learning}. A detailed survey~\cite{frenay2014classification} summarized the early studies on this problem.
In recent years, relying on the development of big data~\cite{zhang2019heterogeneous}, cloud computing~\cite{chen2018label}, and network communications~\cite{zhang2019cocme}, deep learning has made many breakthroughs~\cite{krizhevsky2012imagenet,ren2015faster,long2015fully}. Many successful applications~\cite{chen2019cognitive,zhang2019edge,zhang2018emotion} often employ deep learning technology. Deep neural networks (DNN) have also been used frequently in the field of learning with noisy labels. There are four streams of research within this field, as summarized below.

First, Sukhbaatar et al.~\cite{sukhbaatar2014training} embedded a known noise transition matrix into the loss function. This is a Bayesian method that views real labels as latent variables. Unfortunately, the exact confusion matrix is usually unknown. Later, several methods focused on estimating it~\cite{goldberger2016training,hendrycks2018using,patrini2017making,jindal2016learning}. However, accurate estimations can be difficult to obtain, especially when the number of classes is very large. Moreover, these methods are not suitable for symmetric noise labels.

Second, some approaches were aimed at sample selection to address noisy labels. Decoupling~\cite{malach2017decoupling} and co-sampling~\cite{han2018co} introduced a sample-selection mechanism for carefully training predictors. They both maintain two predictors. The former selected disagreement samples to update the predictors, whereas the latter used small-loss samples to train the predictors. However, the selection mechanism itself is not very reliable because sample-selection bias may cause accumulated errors.

Third, in the context of noise-tolerant methods, several theoretically motivated noise-robust loss functions such as ramp loss and unhinged loss have been introduced~\cite{brooks2011support,van2015learning}. Ghosh et al.~\cite{ghosh2015making,ghosh2017robust} proved and empirically demonstrated that the mean absolute error (MAE) is robust against noisy labels. However, the convergence speed of MAE is slow. Zhang et al.~\cite{zhang2018generalized} found a loss function $\mathcal{L}_q$ that unifies categorical cross entropy (CCE) and MAE to obtain a trade-off relationship between training speed and robustness. Additionally, regularization is an effective method to resist noisy labels, e.g., dropout~\cite{srivastava2014dropout}. Zhang et al.~\cite{zhang2017mixup} proposed Mixup regularization based on the idea that linear interpolations of feature vectors should lead to linear interpolations of the associated targets. Tanaka et al.~\cite{tanaka2018joint} integrated the $\mathcal{L}_p$ regularization~\cite{hu2017learning} and the confidence penalty regularization~\cite{pereyra2017regularizing} into the Kullback--Leibler divergence loss function. However, as DNNs have the characteristic of memorizing noise~\cite{zhang2016understanding}, long-time training leads to performance degradation. Therefore, for noise tolerance, it is important to consider the stop timing in training. These methods often used clean validation samples to find the best epoch, which is similar to early stopping. However, these methods are often sensitive to the validation sets, whose generalization performance is frequently affected by their quality. It is noteworthy that some robust methods based on max-margin do not need clean samples, such as the method reported in~\cite{elsayed2018large}. However, such methods cannot deal with asymmetric label noise (also called class-dependent noise or pair-flipping noise) owing to the limitation of the max-margin.

Fourth, an alternative approach attempts relabeling~\cite{lin2014re}, in which predictors and noisy labels are updated in turns. Bootstrapping~\cite{reed2014training} is a self-learning method with an assumption of consistency. However, the assumption of consistency is not always valid. Further, little attention is given to the selection of the best time to update labels. By using clean samples, the method reported in ~\cite{tanaka2018joint} could be applied to select the suitable stop timing of training and update the labels of the training set for further training. However, as mentioned above, clean samples are not always available, and without clean samples, this method might not be able to choose a suitable stop timing.

To solve this problem, in this work, we propose a weakly supervised learning method without the involvement of clean samples called LGD, which creates a few reverse samples to help estimate the timing of best generalization. LGD is based on the characteristic ~\cite{arpit2017closer} that DNNs tend to learn simple patterns before memorizing noisy patterns. LGD is thus free of models and can be applied to most of the noise-robust methods mentioned above.

\section{Problem Formulation}\label{PF}

\subsection{Polluted Dataset}
Consider the problem of $k$-way multi-class classification. Let $X \subset R^d$ be the feature space and $Y =\{0,1,\cdots,k-1\}$ be the label space. Consider $n$ training data $\{(x_i,y_i^*)\}_{i=1}^n$, where each $(x_i,y_i^*)\in(X \times Y)$ and $y_i^*$ are the true labels of $x_i$, i.e., $y_i^*$ is the oracle. The noisy label $y_i$ is corrupted with respect to the true label $y_i^*$ with the probability $\eta \in(0,1)$. For random sampling from the continuous uniform distribution $U(0,1)$, if the sample falls within $(0,\eta]$, then $y_i=y_{i\lnot}^*$, where $y_{i\lnot}^*$ denotes any defined label except $y_i^*$. If the sample falls within $(\eta,1)$, then $y_i=y_i^*$, i.e., $y_i^*$ is not flipped. Statistically, $\eta$ is the pollution ratio. The labels of the test data are true, i.e., the test labels are the oracle. We further assume that clean validation samples are not available; therefore, we do not establish a validation set.

\begin{figure}[h]
\begin{center}
\includegraphics[width = .6 \textwidth]{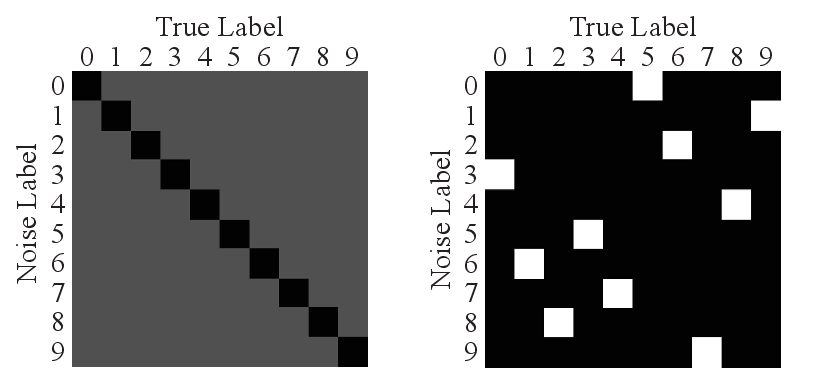}  
\end{center}
    \caption{Left: visualization of the symmetric label noise model. The true labels are flipped to other labels with equal probability.
    Right: example of the asymmetric label noise model. The true labels are flipped to specific false labels according to a fixed rule.}
\label{noise_source}
\end{figure}

\subsection{Pollution Source}
Assume $y_i^*$ is corrupted to $y_{i\lnot}^*$. We consider two sources of pollution: symmetric label noise and asymmetric label noise. The symmetric label noise model obeys a uniform distribution $P(y_i = y_{i\lnot}^* | y_i^* ) = \frac{1}{k-1}$.

The asymmetric label noise model is a specific map $y_i=f(y_i^*)$, $\forall y_i \ne y_i^*$. Here, $f(\cdot)$ can be seen as fixed-rule flipping.
Taking the MNIST dataset as an example, we illustrate the two pollution sources in \Cref{noise_source}.

\section{Prerequisites of Learning with Noisy Labels}

\subsection{Regularity and Scale}\label{RS}
A previous study~\cite{zhang2016understanding} has shown that DNNs have the ability to memorize noise. If a DNN is adequately trained with noisy samples, it will learn not only the main pattern but also noise patterns. This affects the generalization of the main pattern. While the literature~\cite{arpit2017closer} has emphasized that DNNs tend to prioritize the learning of simple patterns over the memorization of noise patterns, our findings extend their claim and further reveal that this simple pattern is the regular pattern with the largest proportion of samples. Here, we should note two important facts: \emph{regularity} and \emph{scale}.

\emph{Regularity} refers to samples that are subject to a certain rule, similar to a case where the label of the character $\mathds{A}$ is 1, that of the character $\mathds{B}$ is 2, and that of the character $\mathds{C}$ is 3. However, if the label of the character $\mathds{A}$ is 2, that of the character $\mathds{B}$ is 3, and that of the character $\mathds{C}$ is 1, the sample follows another rule. We call the latter as label shifting \footnote{Strictly speaking, label shifting refers to cyclic shifting of labels. For the sake of brevity, we still call it label shifting.}. The two rules are mutually exclusive.

\emph{Scale} refers to the number of samples of the regular pattern. In gradient descent optimization, the learning sequences of different regular patterns in samples vary according to their scales.

Assume that two regular patterns are mutually exclusive. One pattern consists of large-scale regular samples (LSRS), whereas the other consists of small-scale regular samples (SSRS). Numerous experiments have confirmed that classifiers will prioritize learning the LSRS rather than the SSRS by optimizing the gradient descent. We observed that, at the early stage of training, the magnitude of gradient accumulation of the LSRS is larger than that of the SSRS. Furthermore, the directions of the two may be quite different. Consequently, the direction of the gradient sum will be biased towards LSRS such that the LSRS learning takes the higher priority. With the progression of training, both the loss and gradient magnitude of the LSRS gradually decrease. When the gradient magnitude of the LSRS decreases to a certain extent, the learning of SSRS will proceed progressively.

For a chaos pattern (e.g., symmetric noise pattern), in general, the scale can be ignored because the scattered gradient directions of the chaos samples lead to a negligible magnitude at the early stage of learning. Therefore, when the chaos pattern exists together with regular patterns, the regular patterns will always be learned first.

To demonstrate this, we first conducted a simple training experiment using the MNIST dataset with the three different patterns mentioned above: the LSRS pattern, the SSRS pattern, and the chaos pattern. Assume that their scales are $N_L$, $N_S$, and $N_C$, respectively, where $N_C=2N_L$ and $N_L=2N_S$. Further, assume that each sample's pattern is known. See \Cref{Tbl_six_datasets} for details. From \Cref{MNIST_exp} (left), we see that the LSRS pattern is learned first at the beginning of training, and its training accuracy increases rapidly, whereas the accuracy of the SSRS pattern does not increase or may even decrease. \Cref{MNIST_exp} (right) shows a 2D plot of the gradients of dimension reduction via t-SNE~\cite{maaten2008visualizing} at this stage. The direction of the gradients' sum is closer to that of the LSRS and deviates from that of the SSRS. Therefore, the learning of SSRS stagnates or even deteriorates. As the training progresses, the accuracy of the SSRS begins to increase gradually, and the speed of increase is greater than that of the chaos pattern. Although the chaos pattern's scale is the largest, its training speed is the lowest because the magnitude of its gradients is too small.

\begin{figure}[h]
\begin{center}
\includegraphics[width = .6 \textwidth]{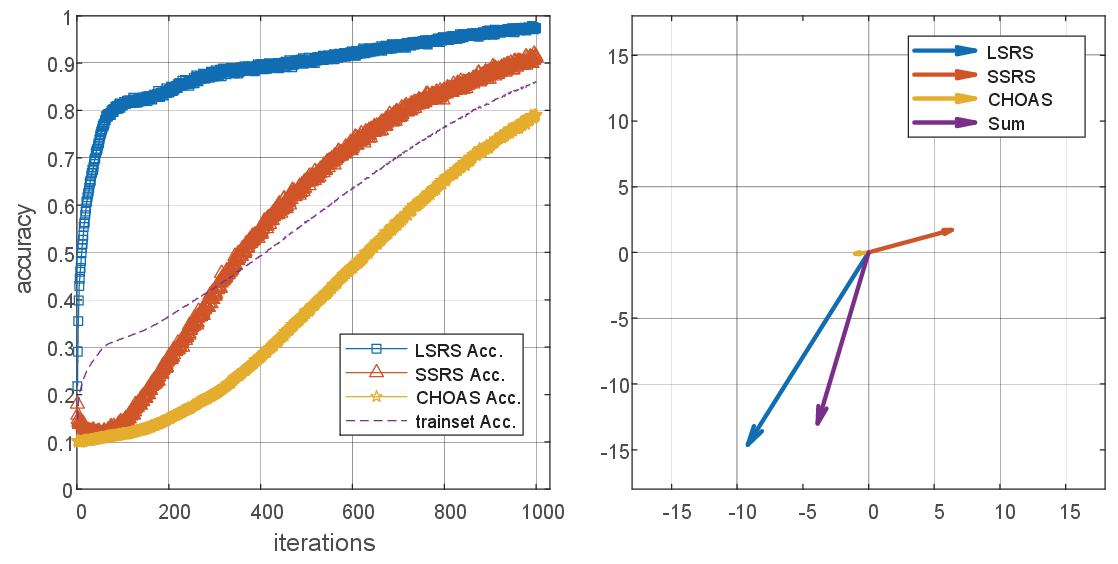}  
\end{center}
    \caption{Left: training accuracy curves. Three different patterns are mixed into a training set. The LSRS is a set of clean samples; the SSRS can be regarded as samples of asymmetric noise labels. The chaos samples can be regarded as samples of symmetric noise labels.
    Right: 2D plot of different patterns' gradients of dimension reduction via t-SNE at the beginning of training.}
\label{MNIST_exp}
\end{figure}


The same trend was also observed with the other more complex datasets. We further performed experiments on six datasets, Fashion-MNIST~\cite{xiao2017fashion}, KMNIST~\cite{clanuwat2018KMNIST}, EMNIST/balanced~\cite{cohen2017EMNIST}, EMNIST/bymerge~\cite{cohen2017EMNIST}, STL10~\cite{coates2011STL10}, and SVHN~\cite{netzer2011SVHN}. The scales of the LSRS, SSRS, and Chaos patterns on the datasets are the same as those in the experimental settings of the MNIST dataset. See \Cref{Tbl_six_datasets} for details. The training results in \Cref{six_datasets} all show that the learning of LSRS will be prioritized during training. From the above experiments, we see that the regularity and scale of the patterns play important roles in training based on gradient descent. For the sequence of learning, regular patterns are prioritized over chaos patterns. Moreover, LSRS patterns are prioritized over SSRS patterns.

\begin{table} [h]
\caption{Pattern-scale settings for the datasets. The training set consists of samples from LSRS, SSRS, and Chaos patterns.}
\scriptsize
\footnotesize
\begin{center}
\begin{tabular}{|c|c|c|c|c|c|}
\hline
Dataset     & trainset     & $N_L$     & $N_S$     & $N_C$    & class num \\
\hline \hline
MNIST           & 56k   & 16k   & 8k    & 32k   & 10\\
\hline
Fashion-MNIST   & 56k   & 16k   & 8k    & 32k   & 10\\
\hline
KMNIST          & 56k   & 16k   & 8k    & 32k   & 10\\
\hline
EMNIST/Balanced & 112k  & 32k   & 16k   & 64k   & 47\\
\hline
EMNIST/ByMerge  & 630k  & 180k  & 90k   & 360k  & 47\\
\hline
STL10           & 12.6k & 3.6k  & 1.8k  & 7.2k  & 10\\
\hline
SVHN            & 70k   & 20k   & 10k   & 40k   & 10\\
\hline
\end{tabular}
\end{center}
\label{Tbl_six_datasets}
\end{table}

\begin{figure}[h]
\begin{center}
\includegraphics[width = .6 \textwidth]{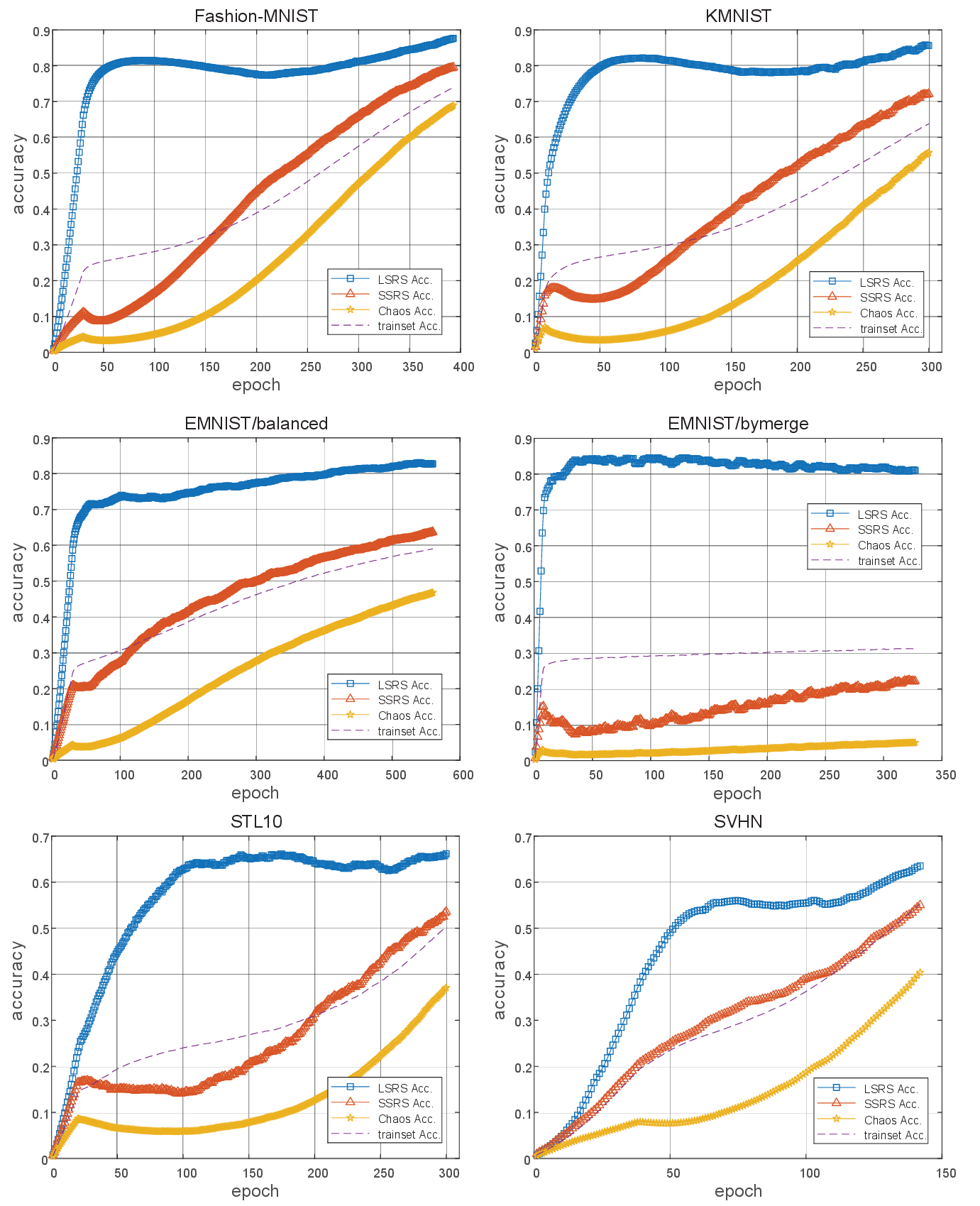}  
\end{center}
    \caption{Training accuracy curves on the six datasets. The three different patterns are mixed into one training set.}
\label{six_datasets}
\end{figure}

Although the training of DNNs has the above characteristics, for the actual training set, one cannot distinguish between the main pattern's samples (clean samples) and the polluted samples. Further, no clean samples exist for validation. Therefore, one cannot determine when the main pattern is best generalized. To solve this problem, some reverse samples are prepared to obtain a reverse pattern that is mutually exclusive with the main pattern. We randomly select $\beta$-proportion samples from the training set to perform label shifting for the reverse pattern. We utilize the reverse pattern to find the best generalization timing of the main pattern. This method is introduced in~\Cref{LGD}. Next, we theoretically provide some necessary conditions of learning with noisy labels and illustrate how to choose the ratio $\beta$ of the artificially reversed samples.

\subsection{Necessary Conditions of Learning with Symmetric Label Noise}\label{NCS}

\begin{lemma}
Consider a $k$-class classification problem. Suppose that the labels of $r$ samples are all polluted by symmetric noise. The label-shifting operation is defined as $\hat{y} = MOD (y+1,k)$ \footnote {MOD is the modulo operation. For instance, $MOD(2,5)=2$ and $MOD(6,5)=1$.}, where $y$ is a polluted label and $\hat{y}$ is the label-shifting result of $y$. Then, after the $r$ labels are shifted, $\frac{r}{k-1}$ samples will attain true labels.
\label{lemma}
\end{lemma}

\begin{proof}
Assume the sample set $P$, the labels of which are all polluted by symmetric noise.
The sample subset with the label $y=j$ is $P_j$, the number of samples of which is $r_j$.
Recall that symmetric noise follows a uniform distribution.
The labels of $P_j$ are distributed along $\{0,1,\cdots,j-1,j+1,\cdots,k-1\}$ with equal probabilities $\frac{1}{k-1}$.
After the labels of $P_j$ are shifted, $\frac{r_j}{k-1}$ samples will attain true labels.

The label shifting of samples with other labels leads to similar conclusions.
After all the labels of $P$ are shifted, the number of samples that attain true labels is

\begin{equation}
\sum_{j=0}^{k-1}\frac{r_j}{k-1} = \frac{1}{k-1}\sum_{j=0}^{k-1}r_j = \frac{r}{k-1}.
\end{equation}
\end{proof}

\begin{theorem}
Suppose the number of samples of set $S$ with noisy labels is $n$, the pollution source of labels is symmetric noise, the pollution ratio is $\eta$, and the number of categories is $k$.
$\beta \cdot n$ samples are randomly selected from $S$ as the reverse pattern via the label-shifting operation, where $\beta\in(0,1)$ is the rate of selection. If the samples with true labels are the main pattern that can be learned first in a noisy environment, then the pollution rate $\eta$ should satisfy $\eta < \frac{k-1}{k}$, and the selection rate $\beta$ should satisfy $\beta < \frac{1-\eta}{2-2\eta- \frac{\eta}{k-1}}$.
\end{theorem}

\begin{table} [h]
\caption{Cross-analysis of the selected and leftover samples with all patterns before label shifting.}
\footnotesize
\begin{center}
\begin{tabular}{|c|c|c|}
\hline
             & Chaos Pattern       & Clean Pattern \\
\hline
Selected & $\eta\beta n$     & $(1-\eta)\beta n$ \\
\hline
Leftover & $\eta(1-\beta) n$ & $(1-\eta)(1-\beta) n$\\
\hline
Total    & $\eta n$           & $(1-\eta) n$ \\
\hline
\end{tabular}
\end{center}
\label{Tbl_SYMM_before_LS}
\end{table}

\begin{table}
\caption{Cross-analysis of the selected and leftover samples with all patterns after label shifting.}
\centering
\setlength{\tabcolsep}{3pt}{
\begin{tabular}{|c|c|c|c|}
\hline
             & Chaos Pattern                    & Regular Shifted Pattern&  Clean Pattern \\
\hline
Selected & $\eta\beta\frac{k-2}{k-1} n$   & $(1-\eta)\beta n$      & $\eta\beta\frac{1}{k-1} n$ \\
\hline
Leftover & $\eta(1-\beta) n$              &   0                      &  $(1-\eta)(1-\beta) n$ \\
\hline
Total    & $ignoring$                           & $(1-\eta)\beta n$      &  $(1-\eta)(1-\beta) n $ \\
         &                                      &                        &  $+ \eta\beta\frac{1}{k-1} n$ \\
\hline
\end{tabular}}
\label{Tbl_SYMM_after_LS}
\end{table}

\begin{proof}
According to the assumptions, we list the numbers of the samples of all patterns before and after label shifting in \Cref{Tbl_SYMM_before_LS} and \Cref{Tbl_SYMM_after_LS}, respectively.
Note that only the labels of the selected samples are shifted.
After the chaos labels of the selected samples are shifted, some of the labels become clean labels (\Cref{lemma}), while the other labels remain as chaos labels. The clean labels of the selected samples attain a regular shifted pattern after label shifting.
The selected samples and leftover samples are separately cross-analyzed with the chaos pattern (symmetric noise pattern), clean pattern (main pattern), and regular shifted pattern.

We need to produce as many reverse samples as possible via the label-shifting operation that are mutually exclusive with the main pattern. Thus, after label shifting, the clean labels of the selected samples should be reduced:
\begin{equation}
\begin{split}
(1-\eta)\beta n > \eta\beta\frac{1}{k-1} n\\
\Rightarrow\eta < \frac{k-1}{k}.\\
\end{split}
\end{equation}

To make the scale of the clean pattern larger than that of the regular-shifted pattern,

\begin{equation}
\begin{split}
(1-\eta)(1-\beta) n + \eta\beta\frac{1}{k-1} n > (1-\eta)\beta n\\
\Rightarrow\beta < \frac{1-\eta}{2-2\eta- \frac{\eta}{k-1}}.\\
\end{split}
\end{equation}

\end{proof}

This work considers the learnability of samples with symmetric noisy labels from the perspective of creating a reverse pattern and attains the condition $\eta < \frac{k-1}{k}$, which is exactly the same result as in~\cite{ghosh2017robust}.
For a symmetric noise source, we suppose that $\eta < \frac{k-1}{k}$ might be the most relaxed condition of learning with noisy labels.

However, $\beta < \frac{1-\eta}{2-2\eta- \frac{\eta}{k-1}}$ is only a basic condition. If the selection rate $\beta$ of the reverse samples is close to the upper bound $\frac{1-\eta}{2-2\eta- \frac{\eta}{k-1}}$, it is actually very difficult to train successfully.

Empirically, we need a tighter condition of $\beta$ to ensure sufficient learning performance.

\begin{proposition}
Following Theorem 1, we further assume that the scale of the clean pattern is not less than $\delta$ times that of the reverse pattern. Then, the selection rate $\beta$ should satisfy $\beta \le \frac{1-\eta}{(1+\delta)(1-\eta)- \frac{\eta}{k-1}}$.
\end{proposition}

\begin{proof}

Similar to Theorem 1, according to the assumptions, the following inequality should be met:

\begin{equation}
(1-\eta)(1-\beta) n + \eta\beta\frac{1}{k-1} n \ge \delta(1-\eta)\beta n.
\end{equation}
Hence,
$$\beta \le \frac{1-\eta}{(1+\delta)(1-\eta)- \frac{\eta}{k-1}}.$$
\end{proof}

When $\delta=1$, this theorem becomes Theorem 1. When $\delta$ is sufficiently large, $\beta \le \frac{1}{1+\delta}$ approximately. From practical experience, the value $\delta$ is set to $\delta\ge9$. Then, $\beta \leq \frac{1}{10}$.

\subsection{Necessary Conditions of Learning with Asymmetric Label Noise}\label{NCA}

\begin{table} [b]
\caption{Cross-analysis of the selected and leftover samples with all patterns before label shifting.}
\footnotesize
\begin{center}
\begin{tabular}{|c|c|c|}
\hline
             & Polluted Pattern       & Clean Pattern \\
\hline
Selected & $\eta\beta n$     & $(1-\eta)\beta n$ \\
\hline
Leftover & $\eta(1-\beta) n$ & $(1-\eta)(1-\beta) n$\\

\hline
\end{tabular}
\end{center}
\label{Tbl_ASYM_before_LS}
\end{table}

\begin{table}
\caption{Cross-analysis of the selected and leftover samples with all patterns after label shifting.}
\scriptsize
\centering
\setlength{\tabcolsep}{3pt}{
\begin{tabular}{|c|c|c|c|c|}
\hline
             & Polluted                          & Shifted Polluted     &  Shifted Clean       & Clean  \\
             & Pattern                           &  Pattern             &  Pattern             &  Pattern \\

\hline
Selected & 0                                          & $\eta\beta n$      & $(1-\eta)\beta n$ &  almost 0  \\
\hline
Leftover & $\eta(1-\beta) n$              &     0                          &   0                           & $(1-\eta)(1-\beta) n$ \\
\hline
\end{tabular}}
\label{Tbl_ASYM_after_LS}
\end{table}

\begin{theorem}
Suppose that the number of samples of set $S$ with noisy labels is $n$, the pollution source of labels is asymmetric noise, the pollution ratio is $\eta$, and the number of categories is $k$. $\beta \cdot n$ samples are randomly selected from $S$ as the reverse pattern via the label-shifting operation, where $\beta\in(0,1)$ is the rate of selection. If the samples with true labels are the main pattern that can be learned first in a noisy environment, then the pollution rate $\eta$ should satisfy $\eta<\frac{1}{2}$, and the selection rate $\beta$ should also satisfy $\beta<\frac{1}{2}$.

\end{theorem}
\begin{proof}

According to the assumptions, we list the numbers of the samples of all patterns before and after label-shifting in \Cref{Tbl_ASYM_before_LS} and \Cref{Tbl_ASYM_after_LS}, respectively.
Recall that the asymmetric polluted samples belong to a regular pattern.
After label shifting, the original polluted samples in the selected samples will form a new regular pattern, called the shifted polluted pattern, and the original clean samples in them will become another new regular pattern, called the shifted clean pattern.
In this case, there are four patterns in the samples set $S$, which are the polluted pattern (asymmetric noise pattern), shifted polluted pattern, clean pattern (the main pattern), and shifted clean pattern.
Owing to the asymmetric pollution, all the samples are of regular patterns. Thus, the scales of all the patterns must be considered.

Note that we further assume that the selected samples will attain the true labels with an extremely small probability after label shifting.
Therefore, after the selected labels are shifted, the number of polluted samples that became clean samples is almost 0.

Recall that the scale of the main pattern is larger than those of all the other regular patterns.
Therefore, we obtain the following simultaneous inequalities:

\begin{subequations}
\begin{empheq}[left=\empheqlbrace]{align}
     (1-\eta)(1-\beta) n > \eta(1-\beta) n\\
     (1-\eta)(1-\beta) n > (1-\eta)\beta n \\
     (1-\eta)(1-\beta) n > \eta\beta n
\end{empheq}
\end{subequations}

Inequality (5a) $\Rightarrow \eta<\frac{1}{2}$.

Inequality (5b) $\Rightarrow \beta<\frac{1}{2}$.

Inequality (5c) naturally holds when $\eta<\frac{1}{2}$ and $\beta<\frac{1}{2}$ hold.
\end{proof}

Theorem 2 establishes a necessary condition that the asymmetric polluted samples should be less than half of the total samples,
whereas $\beta<\frac{1}{2}$ is only a basic condition.
If the selection rate $\beta$ is close to the upper bound $\frac{1}{2}$, it is very difficult to train successfully.
Empirically, we need a tighter upper bound of $\beta$ to ensure sufficient learning performance.

\begin{proposition}
Following Theorem 2, we further assume that the scale of the clean pattern is not less than $\delta$ times those of the reverse patterns and that $\eta<\frac{1}{2}$ holds. Then, the selection rate $\beta$ should satisfy $\beta \le \frac{1}{1+\delta}$.
\end{proposition}

\begin{proof}
Recall that the reverse samples contain two regular patterns: shifted polluted pattern and shifted clean pattern.

Similar to Theorem 2, we have the following simultaneous inequalities:

\begin{subequations}
\begin{empheq}[left=\empheqlbrace]{align}
     (1-\eta)(1-\beta) n \ge \delta(1-\eta)\beta n\\
     (1-\eta)(1-\beta) n \ge \delta\eta\beta n
\end{empheq}
\end{subequations}

Inequality (6a) $\Rightarrow \beta \le \frac{1}{1+\delta}$.

Inequality (6b) holds when $\beta \le \frac{1}{1+\delta}$ and $\eta < \frac{1}{2}$ hold.
\end{proof}

The condition of the selection rate $\beta$ is similar to that in Proposition 1.
When $\delta=1$, this theorem becomes Theorem 2.
From practical experience, the value $\delta$ is set to $\delta \ge 9$. Then, $\beta\leq \frac{1}{10}$.

\section{Limited Gradient Descent}\label{LGD}

To solve the classification problem with noisy labels, it is necessary to know the type of noise source and estimate the pollution ratio~\cite{ramaswamy2016mixture}. If the pollution ratio $\eta$ satisfies the prerequisites of learning with noisy labels (refer to~\Cref{NCS,NCA}), the proposed LGD method can be used for learning. We randomly select $\beta$-proportion samples from the training set to perform label shifting to create the reverse pattern. We can utilize the reverse pattern to help estimate the best generalization timing of the main pattern.

\begin{algorithm} [h]
\caption{Limited Gradient Descent}
Randomly select $\beta \cdot n$ samples from the original training set $S$ to shift the labels as the subset $S_r$, and the leftover samples are referred to as the other subset $S_l$. The new training set now becomes $S' = S_r \cup S_l$.
\begin{algorithmic}
\REQUIRE $Net$, loss function, the training set $S'$, $LoR \leftarrow 0$, LGD iterations $N$
\FOR{each $i \in [1,N]$}
\STATE Train $Net$ by one step (e.g., one epoch) with SGD on $S'$\
\STATE Infer $S_l$ and $S_r$ and obtain the test accuracies $Acc_l$ and $Acc_r$, respectively.\
\IF{$\frac{Acc_l}{Acc_r} > LoR$}
\STATE $LoR \leftarrow \frac{Acc_l}{Acc_r}$
\STATE $net\_rec\leftarrow Net$\
\ENDIF
\ENDFOR
\STATE Predict the labels of the test set with $net\_rec$ to calculate test accuracy.\
\end{algorithmic}
\end{algorithm}

Algorithm 1 illustrates the LGD method. According to the characteristic that DNNs tend to prioritize the learning of the LSRS pattern (see~\Cref{RS}), the main pattern will be learned first if the scale of the main pattern is dominant. The reverse and main patterns are mutually exclusive. We can estimate the generalization performance of the main pattern by observing the training precisions of the leftover and reverse samples. The accuracy of the leftover samples is approximated to that of the main pattern. Meanwhile, the accuracy of the reverse samples is approximated to that of other regular patterns. Training should be stopped when the main pattern is generalized as much as possible and the learning of other regular patterns is suppressed. We design a \emph{leftover-over-reverse} (LoR) metric to estimate the learning performance of the main pattern. When the LoR reaches its maximum value, the main pattern might be best generalized. See Algorithm 1 for details. Because LGD is free of models, it can be applied to most DNNs and loss functions based on SGD optimization. Therefore, the model and loss function are not specified in Algorithm 1.

It is worth mentioning that LGD is different from the methods reported in~\cite{malach2017decoupling} and ~\cite{han2018co}, which only learn reliable samples generally based on confidence or loss. In fact, they are based on a relatively tight learning condition. However, the relatively tight condition could sometimes be difficult to maintain. In other words, selecting reliable samples is sometimes less reliable. Our method relies on the relatively relaxed condition that the main pattern is dominant in the number of samples, rather than on reliable samples.

\section{Experiments}

\subsection{Datasets}
We perform experiments on three datasets, CIFAR-10, CIFAR-100, and Clothing-1M.

\textbf{CIFAR-10/CIFAR-100:} The CIFAR-10 dataset consists of 60000 32$\times$32 color images in 10 classes, with 6000 images per class. There are 50000 training images and 10000 test images. The CIFAR-100 dataset also consists of 60000 32$\times$32 color images in 100 classes, with 600 images per class. There are 50000 training images and 10000 test images as well.

Although LGD does not involve validation sets, corresponding early-stopping-based approaches rely upon them. For comparison, we randomly select several samples from the original training set as a clean validation set. The remaining training samples are then injected with noise to obtain a noisy-label training set. There are two types of noise pollution: symmetric noise pollution and asymmetric noise pollution. See~\Cref{PF} for the pollution operations.

\textbf{Clothing-1M:} Clothing-1M~\cite{xiao2015learning} is a large-scale real-world dataset with over one million images of clothing. The dataset is built by crawling images from several online shopping websites. The labels are generated using the surrounding texts of images provided by the sellers and therefore contain massive errors. The pollution rate of this dataset is estimated to be 38.46\%\footnote {~\cite{xiao2015learning} reported that the correct rate of the labels is 61.54\%.} by~\cite{xiao2015learning}. The type of pollution might be dominated by asymmetric noise, which often exists between some confusion pairs of categories, e.g., Knitwear and Sweater. Clothing-1M dataset thus contains the following 14 classes: T-shirt, Shirt, Knitwear, Chiffon, Sweater, Hoodie, Windbreaker, Jacket, Down Coat, Suit, Shawl, Dress, Vest, and Underwear. The Clothing-1M dataset contains one million training samples with noisy labels and 10,526 testing samples. Although the dataset also contains 47,570 clean training data and 14,313 clean validation data, we do not use them, because LGD neither relies on clean samples for validation nor requires them to aid training.

\subsection{Conditions of Learning with Noisy Labels}
We learned the polluted training sets with LGD to illustrate the effectiveness of the proposed method.

For the CIFAR-10 dataset, we first determine the learnable conditions. For symmetric noise, because the class number $k$ of the CIFAR-10 dataset is 10, according to the conclusion of $\eta<\frac{k-1}{k}$ in Theorem 1, the injected-noise proportion $\eta<0.9$. For asymmetric noise, according to the conclusion of Theorem 2, the injected-noise proportion $\eta<0.5$. In practice, we cautiously set the pollution rate $\eta$ of the symmetric noise so as not to exceed 0.8 and the pollution rate $\eta$ of the asymmetric noise so as not to exceed 0.4.
For the CIFAR-100 dataset, we let the setting of $\eta$ be the same as that of the experiment on the CIFAR-10 dataset.

For the Clothing-1M dataset, as noted above, the pollution ratio is 38.46\% and the class number is 14. According to Theorem 1 and Theorem 2, a pollution rate of 38.46\% satisfies the learnable conditions for either symmetric or asymmetric noise pollution. Therefore, the noisy Clothing-1M dataset is learnable.

In the LGD method, the ratio $\beta$ of the reverse samples is the only hyper-parameter.
We search for the most suitable value of $\beta$ from 0.01 to 0.15 in increments of 0.01 in experiments.
We found that the experimental performance of $\beta=0.05$ is the best, that is, it does not introduce too much noise and retains the smoothness of the LoR curve. In all the experiments, we fixed $\beta=0.05$.

\subsection{Assessment on Polluted CIFAR-10 Dataset}
The experiments illustrate the performance of LGD on the polluted CIFAR-10 dataset.

Because DNNs tend to prioritize the learning of the main pattern over the memorization of noise patterns, most traditional noise-robust methods require a clean validation set to determine when to stop training. However, such traditional methods may be sensitive to validation sets. In reality, making a high-quality validation set might be another problem worth exploring. The LGD method was proposed to consider the situation where clean validation sets are not involved. We embed traditional methods into the LGD framework to adapt to such a situation.

The purpose of the experiments is to compare the two methods; in other words, given a ROBUST method, we compare the performance of ROBUST with ROBUST+LGD. Note that the comparison is in the same environment, i.e., the same network structures and hyper-parameters.

All networks used PreAct ResNet-18~\cite{he2016deep,Preact} with dropout (0.3)~\cite{srivastava2014dropout}. We used CCE loss, noise-tolerance $\mathcal{L}_q$~\cite{zhang2018generalized} loss, and noise-robust Mixup~\cite{zhang2017mixup} regularization as the comparison methods. Consistent with the setting of ~\cite{zhang2018generalized}, the hyper-parameter $q$ in $\mathcal{L}_q$ loss function was set to 0.7. According to ~\cite{zhang2017mixup}, the hyper-parameter $\alpha$ in the Mixup regularization was fixed at 8. A mini-batch size of 128 was used. The learning rate was 0.1. Each reported experiment was repeated 35 times, and noisy labels were randomly generated each time. Unless otherwise specified, the parameter settings in the following CIFAR-100 experiments are unchanged.

To test the impact of the validation set on the experimental results, we take different numbers of validation sets for comparison. For fairness, the number of training samples is fixed at 45,000. The number of training sets in all experiments (also including LGD) are the same. Each time, 45,000 samples are randomly selected from the original training set. Next, in the remaining samples, the required number of samples for each experiment are randomly selected as the validation set. We set the number of validation sets to 100, 200, 500, 1000, 2000, and 5000. Then, the training samples are polluted by $\eta$-rate symmetric or asymmetric label noise. For symmetric noise pollution, we let $\eta$ take values of 0.2, 0.4, 0.6, and 0.8; for asymmetric labels, we let $\eta$ take values of 0.1, 0.2, 0.3, and 0.4. The experimental results are shown in \Cref{CIFAR10-bars}.

\begin{figure*}[h]
\begin{center}
\includegraphics[width = 0.9 \textwidth]{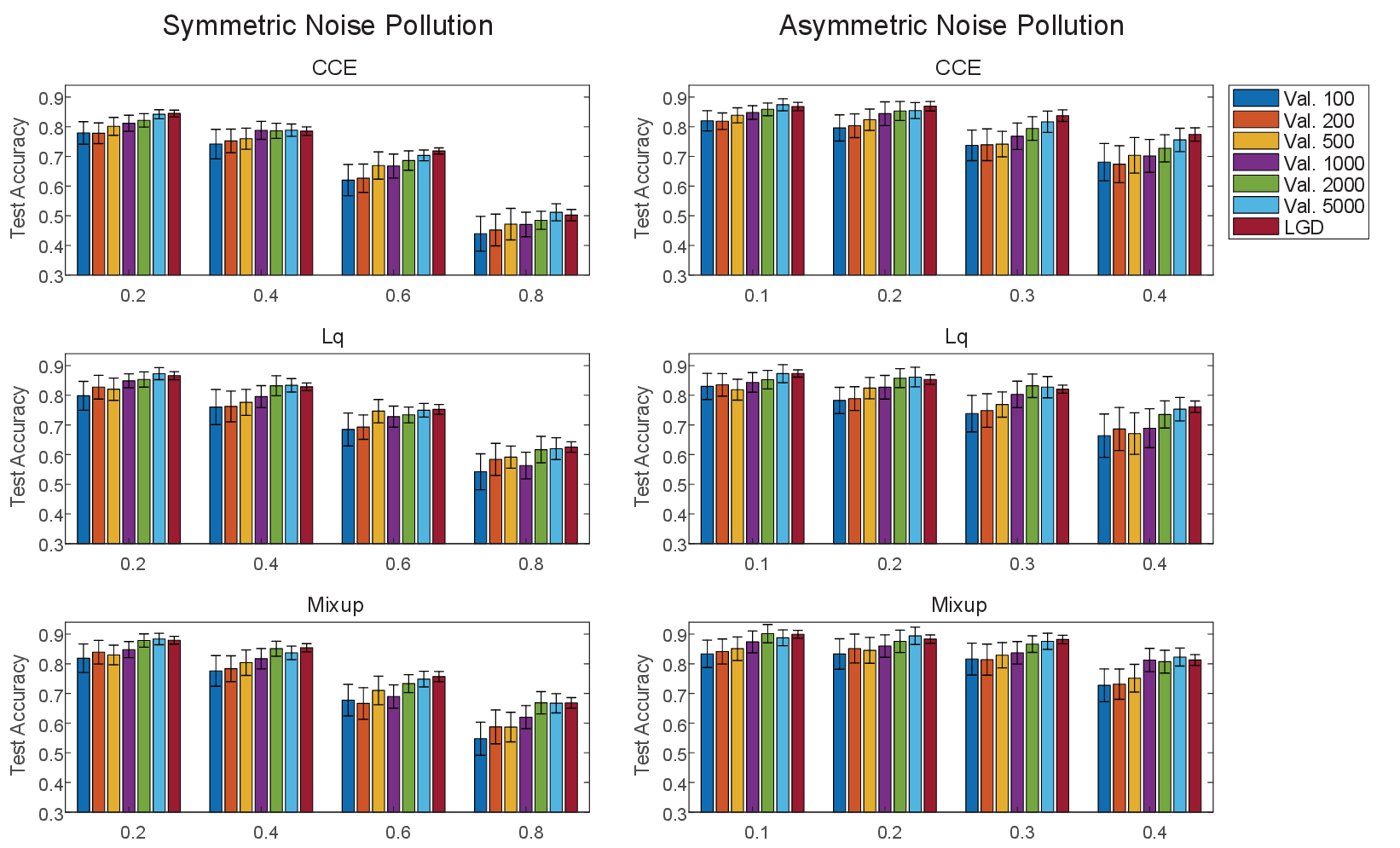}  
\end{center}
    \caption{Comparisons of LGD with the corresponding traditional methods using different numbers of validation sets, which are set to 100, 200, 500, 1000, 2000, and 5000 samples. Each experiment is repeated 35 times. Left: experimental results of symmetrical noise pollution. The pollution rate $\eta$ is set to 0.2, 0.4, 0.6, and 0.8. Right: experimental results of asymmetric noise pollution. The pollution rate $\eta$ is set to 0.1, 0.2, 0.3, and 0.4.}
\label{CIFAR10-bars}
\end{figure*}

As can be seen from \Cref{CIFAR10-bars}, when the number of validation sets increases from 100 to 5000, the test accuracy shows an increasing trend overall. At the same time, as the number of validation sets increases, the variance of overall test accuracy gradually decreases. It shows that these kinds of methods are sensitive to validation sets. The number and quality of validation sets can affect the test accuracy and its variance based on the method. \Cref{CIFAR10-bars} also demonstrates the robustness of our approach, which does not require a validation set. The LGD accuracies are overall comparable to the corresponding methods that require a validation set of 5000 samples. At the same time, the variances of the test accuracies are significantly smaller than those of the conventional methods, even when compared to the methods with 5000 validation samples. Therefore, LGD avoids involving a validation set and is more robust than the traditional method.

Next, we show training details of the LGD through two experiments. We take mixup vs. LGD + mixup as an example of comparison. The pollution source is symmetric noise, and the pollution fraction $\eta$ is 0.6 in the first experiment, as shown in \Cref{LoR} (left). For the second experiment, as shown in \Cref{LoR} (right), the pollution source is asymmetric noise, and the pollution rate $\eta$ is 0.3. In \Cref{LoR} (top), the three accuracy curves correspond to the leftover samples, reverse samples, and test samples. The LoR curves are shown in \Cref{LoR} (bottom). In general, the peak of LoR is located on the left of the peak of the test accuracy curve but not far from it, which can often be regarded as the best generalization of the main pattern. The test accuracy corresponding to the LoR peak is not much different from the best accuracy of the test curve. Therefore, the generalization performance of LGD, which is indicated by the LoR peak, is very close to the actual best generalization performance.

\begin{figure}[t]
\begin{center}
\includegraphics[width = .6 \textwidth]{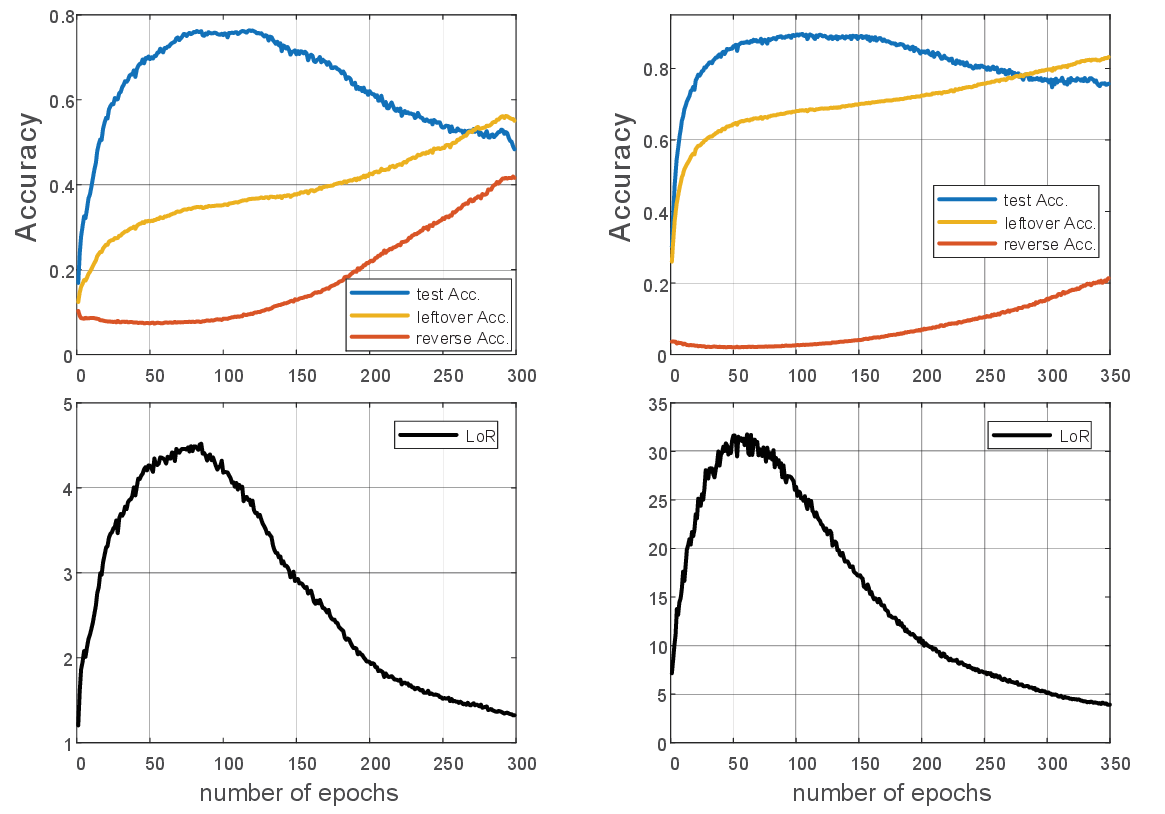}  
\end{center}
    \caption{\textbf{Left:} example of learning with symmetric noisy labels ($\eta=0.6$).
    \textbf{Right:} example of learning with asymmetric noisy labels ($\eta=0.3$). }
\label{LoR}
\end{figure}

\subsection{Assessment on Polluted CIFAR-100 Dataset}
The experiments illustrate the performance of LGD on the polluted CIFAR-100 dataset.

In this experiment, we compare LGD with the noise-robust method that requires a validation set, namely ROBUST vs. ROBUST+LGD. The compared methods are the same as the previous CIFAR-10 experiment, i.e., CCE, $\mathcal{L}_q$ and Mixup. We randomly select 5000 samples in the original training set as a clean validation set, leaving 45,000 samples to inject $\eta$-proportion symmetric or asymmetric label noise as a noisy trainset. The traditional method uses the noisy training set and the clean validation set, while the LGD method only uses the noisy training set. For symmetric noise pollution, we let $\eta$ be 0.2, 0.4, 0.6, and 0.8; for asymmetric labels, we let $\eta$ be 0.1, 0.2, 0.3, and 0.4. For fairness, the noisy training sets used by a pair of compared methods are the same. The test results are evaluated on the testsets, which are listed in \Cref{Tbl_CIFAR100}.

\begin{table*}[h]
\caption{Average test accuracy ($\%$) on the CIFAR-100 dataset (35 runs). The experiment does not compare the performance of different robust methods, but rather compares LGD with the corresponding traditional methods. Therefore, there is no boldface to mark the best accuracies.}
\begin{center}
\begin{tabular}{|c|c|c|c|c|c|c|c|c|}
\hline
    &  \multicolumn{8}{|c|}{Symmetric Noise Pollution} \\
\hline
    &  \multicolumn{2}{|c|}{0.2} & \multicolumn{2}{|c|}{0.4} & \multicolumn{2}{|c|}{0.6} & \multicolumn{2}{|c|}{0.8} \\
\hline
               & validation & LGD & validation & LGD & validation & LGD & validation & LGD\\
\hline
CCE            & 55.06 & 54.99 & 46.53 & 47.89 & 34.41 & 35.33 & 17.73 & 20.26\\
\hline
$\mathcal{L}_q$             & 64.53 & 64.26 & 56.40 & 56.42 & 42.13 & 45.90 & 25.87 & 28.10\\
\hline
Mixup          & 63.98 & 62.52 & 59.90 & 59.64 & 46.85 & 45.78 & 27.90 & 29.35\\
\hline
    &  \multicolumn{8}{|c|}{Asymmetric Noise Pollution} \\
\hline
    &  \multicolumn{2}{|c|}{0.1} & \multicolumn{2}{|c|}{0.2} & \multicolumn{2}{|c|}{0.3} & \multicolumn{2}{|c|}{0.4} \\
\hline
               & validation & LGD & validation & LGD & validation & LGD & validation & LGD\\
\hline
CCE            & 62.07 & 61.67 & 59.52 & 59.08 & 56.34 & 58.78 & 50.45 & 53.11\\
\hline
$\mathcal{L}_q$             & 64.80 & 64.17 & 63.25 & 64.45 & 61.12 & 62.25 & 57.22 & 60.43\\
\hline
Mixup          & 67.70 & 68.94 & 66.38 & 65.85 & 65.35 & 65.72 & 63.96 & 65.57\\
\hline
\end{tabular}
\end{center}
\label{Tbl_CIFAR100}
\end{table*}

For symmetrical noise pollution experiments, the test accuracy of LGD method is overall equivalent to that of the corresponding traditional method at a pollution rate of 0.2 to 0.6. At the large pollution rate of 0.8, the test results of LGD are all higher than those of the corresponding traditional methods.
The similar trend has also occurred in asymmetric noise pollution experiments. The test accuracy of the LGD method is overall equivalent to that of the corresponding method at a pollution rate of 0.1 and 0.2. At the pollution rate of 0.3 and 0.4, the test results of LGD are higher than the corresponding traditional methods.

\subsection{Assessment on Clothing-1M Dataset}
Finally, we performed the experiments on the Clothing-1M dataset~\cite{xiao2015learning} to test the deployment ability of LGD in real-world applications and further compare the performance of LGD with traditional methods.

Since Clothing-1M is a large-scale dataset, we consider using the model Resnet-50 pretrained on ImageNet, which was also used in ~\cite{patrini2017making}. We still use three traditional loss functions such as CCE, $\mathcal{L}_q$, and Mixup as the comparison methods relying on clean validation sets. Then, we embedded these methods into the LGD training framework. That is, we practiced six methods on the Clothing-1M dataset: CCE, $\mathcal{L}_q$, Mixup, LGD+CCE, LGD+$\mathcal{L}_q$ and LGD+Mixup.

We only used the noisy trainset of one million samples and evaluated the scores on the testset. The traditional noise-robust methods require clean validation sets. We also show the results reported in\cite{patrini2017making, tanaka2018joint}. Tanaka\cite{tanaka2018joint} used the clean validation set to determine the stop timing of the training, then updated the label of the trainset and alternated several times. We did not used clean validation data. We only used noisy training data. The experimental results are listed in \Cref{Tbl_Clothing1M}.

\begin{table}[h]
\caption{Test accuracies (\%) on Clothing-1M dataset. \#1 and \#2 are quote from according literature. \#3, \#4 and \#5 are the results of traditional methods by our implementation. \#6, \#7 and \#8 are the results of our methods. }
\footnotesize
\begin{center}
\begin{tabular}{|c|c|c|c|c|}
\hline
\# & Method & validation set  & Test Accuracy \\
\hline \hline
1 & Forward\cite{patrini2017making} & \checkmark & 69.84 \\
\hline
2 & Tanaka\cite{tanaka2018joint} & \checkmark & 72.23 \\
\hline
3 & CCE & \checkmark & 68.87 \\
\hline
4 & Mixup & \checkmark & 71.35 \\
\hline
5 & $\mathcal{L}_q$ & \checkmark & 71.91 \\
\hline \hline
6 & LGD+CCE & - & 69.65 \\
\hline
7 & LGD+Mixup & - & 73.67 \\
\hline
8 & LGD+$\mathcal{L}_q$ & - & \textbf{74.36} \\
\hline
\end{tabular}
\end{center}
\label{Tbl_Clothing1M}
\end{table}

From \Cref{Tbl_Clothing1M}, we can see that the test accuracies of LGD+CCE, LGD+$\mathcal{L}_q$, LGD+Mixup are higher than the corresponding CCE, $\mathcal{L}_q$ and Mixup methods, respectively. The LGD+$\mathcal{L}_q$ method achieves the highest test accuracy, which is 2.45\% higher than the corresponding $\mathcal{L}_q$ method. The experimental results also show that in practical applications (real-world datasets), if a method relies on the validation set, it might be at risk. Because many methods which rely on validation sets may indeed be sensitive to validation sets, and the quality of the validation sets may affect the final test accuracy. The proposed method does not involve the use of validation sets and therefore eliminates the associated risks.

\section{Conclusion}
This paper presented the LGD method for learning with noisy labels. LGD is based on an interesting characteristic whereby DNNs tend to prioritize the learning of an LSRS pattern over the learning of an SSRS pattern or even a chaos pattern. Traditional methods often use a clean validation set to monitor the training process. In contrast, LGD does not require a clean validation set; it creates a few samples that are different from the main pattern to help estimate the learning progress of the main pattern. It works under a quite relaxed condition that the scale of the main pattern is dominant in samples. We empirically verified the effectiveness of the algorithm on various datasets. The experimental results on CIFAR-10, CIFAR-100 and Clothing-1M datasets support the practical application of LGD.

\bibliography{reference}
\bibliographystyle{unsrt}

\end{document}